\documentclass[]{siamart220329}

\pdfpagewidth=\paperwidth
\pdfpageheight=\paperheight
\usepackage[scr=boondox]{mathalfa}
\usepackage{amsmath,graphicx}
\usepackage{verbatim}
\usepackage{color}
\usepackage{amsfonts,amsbsy,amssymb}
\usepackage{natmove}
\usepackage{enumitem}
\setcitestyle{numbers}

\def\br{\boldsymbol {\rm r}}

\def\d{\!\!{\rm d}}
\def\br{\boldsymbol {\rm r}}

\newcounter{concount}
\newcounter{exampcount}
\newcounter{algcount}

\newsiamremark{remark}{Remark}
\newsiamremark{remarks}{Remarks}
\newsiamremark{keyremark}{{\bf Key Remark}}
\newsiamremark{note}{{\bf Note on Terminology}}
\newsiamremark{notation}{{\bf On Notation}}
\newsiamremark{example}{Example}
\newsiamthm{postulate}{Postulate}
\newsiamremark{construction}{Construction}
\title{Analysis of Diagnostics (Part II):  Prevalence, Linear Independence, \& Unsupervised Learning}

\author{Paul N.\ Patrone\footnotemark[1] 
\and Raquel A.\ Binder\footnotemark[2]
\and Catherine S.\ Forconi\footnotemark[2]
\and Ann M.\ Moormann\footnotemark[2]
\and Anthony J.\ Kearsley\footnotemark[1]
}

\begin{document}

\renewcommand{\thefootnote}{\fnsymbol{footnote}}

\footnotetext[1]{National Institute of Standards and Technology, Gaithersburg MD 20899, USA}
\footnotetext[2]{Department of Medicine, Division of Infectious Diseases and Immunology, University of Massachusetts Chan Medical School, Worcester,
MA, 01655, USA}

\maketitle

\begin{abstract}
This is the second manuscript in a two-part series that uses diagnostic testing to understand the connection between prevalence (i.e.\ number of elements in a class), uncertainty quantification (UQ), and classification theory.  Part I considered the context of supervised machine learning (ML) and established a duality between prevalence and the concept of relative conditional probability.  The key idea of that analysis was to train a \textit{family of discriminative classifiers} by minimizing a sum of prevalence-weighted empirical risk functions.  The resulting outputs can be interpreted as relative probability level-sets, which thereby yield uncertainty estimates in the class labels.  This procedure also demonstrated that certain discriminative and generative ML  models are equivalent.  Part II considers the extent to which these results can be extended to tasks in unsupervised learning through recourse to ideas in linear algebra.  We first observe that the distribution of an \textit{impure population}, for which the class of a corresponding sample is unknown, can be parameterized in terms of a prevalence.  This motivates us to introduce the concept of \textit{linearly independent populations}, which have different but unknown prevalence values.  Using this, we identify an isomorphism between classifiers defined in terms of impure and pure populations.  In certain cases, this also leads to a nonlinear system of equations whose solution yields the prevalence values of the linearly independent populations, fully realizing unsupervised learning as a generalization of supervised learning.  We illustrate our methods in the context of synthetic data and a research-use-only SARS-CoV-2 enzyme-linked immunosorbent assay (ELISA). 
\end{abstract}

\section{Appeal to the Reader}

This is the second manuscript in a two-part series that uses examples and ideas from diagnostics to study the relationship between probability, classification, and uncertainty quantification (UQ).  In our experience, the connection between these four fields is anchored by a key concept in epidemiology, namely the \textit{prevalence} or the fraction of individuals with a given condition (e.g.\ an infection).  In public health settings, this quantity is often of inherent interest because it informs, for example, strategies for preventing the spread of a disease.  But somewhat counterintuitively, prevalence also controls the accuracy of diagnostic tests, especially in situations for which there is no so-called ``gold standard'' for detecting a medical condition.  

Despite its central role in diagnostics, the concept of prevalence appears to have been largely overlooked in the broader fields of machine learning (ML) and classification theory.  Motivated by this observation, Part I explored how prevalence leads to a new interpretation of discriminative-like classifiers as relative probability level-sets, which thereby yield  uncertainty estimates in the class labels.  Perhaps more importantly, this perspective unifies certain types of generative and discriminative ML models by demonstrating that they are mathematically equivalent.  This also reveals that certain tasks in supervised learning can be rigorously recast as statistical regression.  Part II extends these results by making new connections between unsupervised learning and linear algebra.  Our main thesis is that certain types of unsupervised training generalize supervised learning in the same way that linear independence generalizes orthogonality.

In developing its core theses, Part I adopted a perspective rooted in epidemiology and diagnostics.  We continue with that setting in this exposition.  As a result, this manuscript is interdisciplinary.  A typical reader should have a broad interest in the aforementioned fields of applied mathematics, with a willingness to ground the analysis in examples from \textit{serology (or blood testing)}.  While this manuscript is also intended to stand alone, we felt that complete duplication of material was unwarranted.  Thus, much of the associated context is omitted.  We advise the reader to consult Part I for more in-depth motivation of the methods considered herein.

\section{Introduction: Motivation and Key Questions}
\label{sec:introduction}

%\subsection{Conceptual Motivation and Introduction}

In traditional serology settings, there exists a population or sample space $\Omega$ of individuals $\omega$, some of whom have a specified medical condition, e.g.\ an infection.  Instead of knowing the medical status of each individual, we are given a measurement outcome $\br(\omega)$, which can be interpreted as both a diagnostic test result and a random variable.   Based on $\br(\omega)$ alone, diagnosticians often seek to answer two questions: 
\begin{itemize}%[leftmargin=.5in]
\item[Q.I] which people have the condition;
\item[Q.II] and how many people have the condition?
\end{itemize} 
The first question is clearly a classification problem; the second is the {\it prevalence estimation} problem.\footnote{Prevalence is the fraction of individuals having a condition.  In the context of serology testing, this is more appropriately called ``seroprevalence,'' which is the fraction of individuals having blood-borne markers of the condition.  We use the words prevalence and seroprevalence interchangeably.}

Within this context, there are generally two strategies for answering Question Q.I, and both require the use of \textit{pure training data}, for which the underlying sample classes are known.  Generative strategies use this  data to construct probability models of the measurement outcomes conditioned on the sample classes, which then inform a Bayes optimal classifier, for example \cite{Patrone21_1,Patrone22_1,Patrone22_2,Luke23_2,Bayes,RW}.  Discriminative approaches typically use the training data to construct a classification boundary that partitions the underlying measurement space $\Gamma$ into domains corresponding to positive and negative samples \cite{SARSClassification1,ROC,3sig1,3sig2,3sig3,Gating,OldPrevOpt,EUA,Disc2,Disc3,Hardle}.  Part I of this series demonstrated that the two approaches are mathematical converses when the classifiers are constructed by minimizing a prevalence-weighted empirical risk function \cite{PartI}.

The goal of this manuscript is to generalize the results of Part I by showing that impure training data (for which the underlying true classes are unknown) can be used to construct classifiers that are identical to those associated with pure training data.  The key ideas behind this analysis are that: (i)  impure datasets can be parameterized by a prevalence; and (ii) two such datasets having different prevalence values correspond to a type of linearly independent systems.  This leads to an isomorphism that relates classifiers associated with supervised and unsupervised training.  In this context, we also demonstrate that when the probability models have partially disjoint supports, the unknown prevalence values of the impure populations can be determined by solving a pair of nonlinear equations depending on the distribution measures.  We present numerical techniques for addressing these tasks and illustrate our analysis using synthetic and a SARS-CoV-2 serology assay.  

The motivation for this work arises from the fact that it is often difficult to acquire pure training data in diagnostic settings \cite{Unlabeled1,Unlabeled2,Unlabeled3,Bats1,Bats2}.  This is especially true in wildlife studies and for endemic conditions, such as infection of a common, widely circulating virus (e.g. influenza).  In such cases, it is often possible to find subsets of a population that have a small (unknown) prevalence, but one can rarely guarantee that every ``low'' measurement signal corresponds to a negative sample.  Likewise, for populations with a high prevalence, cross-reactivity to related viral strains can cause false-positives to be incorrectly added to a training population.  Such effects are difficult to address without explicitly accounting for the fact that they make model training an exercise in unsupervised learning.  

A key challenge in this work arises from the fact that in a binary setting with two linearly independent populations, there are in general four unknowns: two distributions that quantify the probability of a measurement outcome conditioned on the sample class, and two prevalence values.  By analogy to linear algebra, it is clear that given only the PDFs of the impure populations, it is impossible to determine all four unknowns uniquely, and in fact, this is easy to prove rigorously.  Thus to make use of the isomorphism between impure and pure training data, it is necessary to have additional information of some kind.  We resolve this  via our assumption of \textit{partially disjoint distributions}, which we justify, for example, through the intuition that the ``most positive sample'' always yields a measurement signal that is greater than any negative sample.  Even then, determining the unknowns requires a nuanced reinterpretation of prevalence and an analysis of certain classifiers as limits of sets.

As in Part I, an important limitation of our analysis is its restriction to a binary setting.  We accept this for two reasons.  First, the binary classification problem is sufficiently rich that, in our opinions, it warrants a study of its own.  But perhaps more importantly, we anticipate that an analysis of the multiclass setting follows by extension of ideas developed herein.  In certain places we point to multiclass extensions of our analysis, but the bulk of such work is left for future manuscripts.  

In comparing to Part I, recognize also that the present manuscript is primarily concerned with limiting cases, e.g.\ in which the probability densities of training populations are known to arbitrary precision.  Part I considered such cases both \textit{per se} and to motivate numerical techniques for analyzing finite datasets.  In doing so, it was necessary to precisely define concepts such as empirical test and training data.  In contrast, our primary goal in this work is to connect supervised and unsupervised learning when the probability models are given.  It turns out that this is sufficient to transfer the numerical techniques of Part I to the present setting.  For this reason, we omit some of the context surrounding empirical data, instead focusing only on those issues that are specific to analysis of unsupervised learning.  We refer the reader to Part I for a full treatment of empirical datasets.

The rest of the manuscript is organized as follows.  Section \ref{sec:setting} provides technical background and overviews the mathematical setting of diagnostic classification.  Section \ref{sec:impure} develops the theory connecting unsupervised and supervised model training.  Section \ref{sec:discussion} discusses our main results in the context of previous works and points to open questions.

\section{Mathematical Setting}
\label{sec:setting}

The purpose of this section is to provide necessary technical background in support of our main results.  Many of the ideas presented herein also appear in Part I, and we refer the reader to that manuscript for additional context.  However, we do not advise skipping this section, despite a degree of overlap.  Several concepts originally formulated in Part I are extended in this section.

\subsection{Notation}
\label{subsec:notation}

We leverage the following conventions motivated by diagnostics.
\begin{itemize}
\item[(a)] Except when referring to the sign of a number, the terms ``negative'' and ``positive'' generally denote some medical condition, e.g.\ an individual having a certain type of antibody in his or her blood sample.
\item[(b)] When used as subscripts, the letters $n$ and $p$, which we often use in place of $0$ and $1$, denote ``negative'' and ``positive'' in the sense of (a).
\item[(c)] When used as subscripts, $l$ and $h$ refer to ``low'' and ``high.''  In particular, if $q_l$ and $q_h$ are prevalence values, the indices indicate that $q_l < q_h$.   
\item[(d)] The capital letters $N$ and $P$ are associated with negative and positive populations in the sense of (a).
\end{itemize}

We also employ the following notation throughout.
\begin{itemize}
\item Bold lowercase Roman and Greek letters (e.g.\ $\br$) denote column vectors; non-bold versions are scalars.
\item Subscripts $i$, $j$, and $k$ attached to vectors denote random realizations of a vector, not its components.
\item The symbols $D$ and $B$  always refer to sets.  $B$ always denotes a ``boundary set'' between two domains.  
\item The notation $X_D$ means
\begin{align}
X_D = \int_{D}\d\br \,\,\, X(\br)
\end{align}
where $X$ is always a probability density function and $D$ is a set.  That is, $X_D$ is the measure of set $D$ with respect to the PDF $X(\br)$.  
\item The acronym {\it iid} always means ``independent and identically distributed.''
%\item When used in the argument of a function, a semicolon separates variables from parameters, the latter being treated as fixed.  For example, $Z(\br;q)$ denotes a function of $\br$, where $q$ should be interpreted as constant.  While this distinction is largely semantic, our analysis eventually reinterprets parameters as quantities that we control.  This has important implications for uncertainty quantification.  
\end{itemize}

\subsection{Key Assumption}
We always assume absolute continuity of measure \cite{Tao}.  See Part I for more context on the usefulness of this assumption.

\subsection{Background Theory}
\label{subsec:bg}

Diagnostic assays are used to determine the properties of a {\it test population}.  This population is often associated with some sample space $\Omega$, and we assume that individuals $\omega\in \Omega$ belong to one of two classes, referred to colloquially as ``negative'' and ``positive.''  We denote this class by the discrete random variable $C(\omega)$.  In diagnostics, the class of an individual in a test population is typically unknown.  In its place, we are given the result $\br(\omega)$ of a diagnostic test, which is a random variable in some space $\Gamma \subset \mathbb R^n$.  The goal of classification is to deduce $C(\omega)$, given only  $\br(\omega)$.  Prevalence estimation determines the fraction of positive individuals in the population given a set of measurements $\{\br(\omega_i)\}$ indexed by $i$.  In either case, we refer to the collection of measurements being analyzed as the {\it test data}.    

\begin{remark}
When there is no risk of confusion, we omit dependence of $\br$ on $\omega$, since the latter is most often the quantity that we can access.  See Part I and Refs.\ \cite{pspace,Evans} for context associated with probability spaces needed to rigorously define our mathematical setting.  
\end{remark}

As a preliminary step, we must construct the probabilistic framework describing the measurement process of a diagnostic test.  This can be accomplished by recourse to the law of total probability \cite{totprob}.  To that end, the following concepts are foundational.

\begin{definition}[Prevalence Convention]
Let $C(\omega)\in \{0,1\}$ be a binary random variable.  We adopt the convention that $C(\omega)=0$ is a ``negative'' sample and $C(\omega)=1$ is a ``positive'' sample.  Moreover, we define the {\bf prevalence} to be the probability $q$ that $C(\omega)=1$.\label{def:prev}
\end{definition}

\begin{definition}[PDF Conventions]
Let $\br(\omega)\in \Gamma \subset \mathbb R^m$ be a random variable for some positive integer $m$.  Also let $C(\omega)$ be a binary random variable.  We refer to $P(\br)$ as the {\bf conditional probability density function} of $\br(\omega)$ conditioned on $C(\omega)=1$.  That is, it is the probability that a positive sample yields measurement value $\br$.  Likewise, $N(\br)$ is the conditional PDF for a negative sample.  
\end{definition}

The previous two definitions lead to the following model of a measurement process.
\begin{definition}[PDF of a Binary Test Population]
Let $C(\omega)$ be a binary random variable, and assume $q \in [0,1]$.  Then
\begin{align}
Q(\br;q)=qP(\br)+(1-q)N(\br)  \label{eq:qdef}
\end{align}
is the PDF that a {\bf test sample} (whose class is unknown) yields measurement $\br$.  We refer to the underlying sample space from which the $\br$ are generated to be a {\bf test population} with prevalence $q$. \label{def:qdef}
\end{definition}

\begin{remark}
For more context on the use of Defs.\ \ref{def:prev}--\ref{def:qdef}, see Part I.  Here it suffices to note that when $q=0$ or $q=1$, $Q(\br;q)$ corresponds to the PDF of a pure negative or pure positive population.  Thus, we also interpret $Q(\br;q)$, $0 < q<1$ to be the PDF of an \textbf{impure population}. 
\end{remark}

Given a sample generated by $Q(\br;q)$, diagnostic classifiers are generally identified with a partition $U=\{D_n,D_p\}$ of the measurement space $\Gamma$.  The statements $\br\in D_p$ and $\br\in D_n$ are interpreted as $\br$ corresponding to a positive or negative sample.  More precisely:
\begin{definition}[Classifier]
Let $U=\{D_n,D_p\}$ be a partition of $\Gamma$ such that
\begin{subequations}
\begin{align}
D_n \cup D_p &= \Gamma, \\
 D_n \cap D_p &= \emptyset.
\end{align} 
\end{subequations}
We refer to $U$ as a (binary) {\bf classifier} and assign $\br$ to the negative or positive class according to $\br \in D_n$ or $\br \in D_p$.  .  
\label{def:classifier}
\end{definition}

\begin{remark}
While not needed in this manuscript, observe that $U$ induces a new random variable $\hat C(\br(\omega),U)$ whose value is $j$ when $\br \in D_j$ (where $D_0=D_n$ and $D_1=D_p$).  Such a perspective establishes connections to ML techniques that do not explicitly represent classifiers in terms of sets.  However, for random variables that can be expressed as vectors, we may safely assume that there exists an underlying representation (whether explicit or not) of the classifier as a partition.  This perspective is particularly useful for making connections to probability.  Note that by assuming continuity of measure, $U$ defines an equivalence class of classifiers that only differ on sets of measure zero.  We always adopt this more general perspective when working with partitions.  See Part I for more context on these issues.
\end{remark}

An important observation of Ref.\ \cite{Patrone21_1} is that the average error rate of a diagnostic assay is a scalar-valued function of a partition.  In  particular:
\begin{lemma}[Binary Classification Error]
Let $U=\{D_n,D_p\}$ be a binary classifier.  The mapping ${\mathcal E: U \to [0,1]}$ defined via
\begin{align}
\mathcal E(U;q) = (1-q)N_{D_p} + qP_{D_n} \label{eq:error}
\end{align}
is the average or {\bf expected classification error} associated with $U$ for a test population having prevalence $q$.  \label{def:classerror}
\end{lemma}

See Part I for a straightforward proof based on the law of total probability \cite{totprob}.

A key objective in classification is to minimize the error associated with assigning class labels.  The following result addresses this issue \cite{PartI}.
\begin{lemma}[Optimal Binary Classifier]
Assume  $q\in [0,1]$.  Then the partition $U^\star$ whose elements are
\begin{subequations}
\begin{align}
D_p^\star(q) &= \{r:qP(\br) > (1-q)N(\br) \} \cup  B_p^\star(q) \label{eq:Dp}\\
D_n^\star(q) &= \{r:qP(\br) < (1-q)N(\br) \} \cup  B_n^\star(q) \label{eq:Dn}
\end{align}
\end{subequations}
minimizes $\mathcal E(U;q)$, where $B_p^\star(q)$ and $B_n^\star(q)$ form an arbitrary partition of the set ${B^\star(q)= \{r:qP(\br) = (1-q)N(\br) \}}$.  Moreover, the converse holds.  Any partition $U^\dagger$ minimizing $\mathcal E(U;q)$ is in the equivalence class defined by $U^\star$. \label{lem:optclass}
\end{lemma}

\begin{keyremark}
Lemma \ref{lem:optclass} plays a critical role in Part I of this series.  We view it as the foundation for UQ of classification because it implies that $\mathcal E(U;q)$ can be used to quantify uncertainty in the class labels.  Moreover, this holds no matter how we construct the classifier, i.e.\ independent of our direct knowledge of $P(\br)$ and $N(\br)$.  In order to connect supervised and unsupervised learning, our main task therefore amounts to identifying an appropriate unsupervised objective function $\mathcal E_I$ that is equivalent to $\mathcal E(U;q)$, so as the guarantee that they have the same minimizer.    
\end{keyremark}

We end this section on a technical point needed to address an issue associated with sets of measure zero.  The following definition ensures that classifiers do not have degenerate sets violating the pointwise structure of $U^\star$.
\begin{definition}
We say that $U^\dagger(q)=\{D_p^\dagger,D_n^\dagger,B^\dagger\}$ satisfies the {\bf optimal partition convention} if $U^\dagger(q)$: (i) is in the equivalence class minimizing $\mathcal E(U;q)$; and (ii) ${\br \in D_p^\dagger(q) \iff qP(\br) \ge (1-q)N(\br)}$ and ${\br \in D_n^\dagger(q) \iff qP(\br) \le (1-q)N(\br)}$.  \label{def:optconv}
\end{definition}
Note that the optimal partition convention does not determine how points in $B^\star$ are assigned to sets in $U^\dagger$.  See Part I for further discussion and context on the optimal partition convention \cite{PartI}.

\section{Impure Training Data}
\label{sec:impure}

\subsection{Linear Independence}

The classifier given by $U^\star$ is impossible to realize in practice without some method for injecting information about $P(\br)$ and $N(\br)$.  In supervised settings, one is given a \textit{training data set} composed of samples whose underlying true classes are known.  These can be used to directly model the conditional PDFs, or alternatively, one can construct a boundary (for example) that best classifies the data in some appropriate sense.  Such issues amount to questions of how to model data.  We refer the reader to Part I for an in-depth analysis \cite{PartI}.  Here we seek to answer a different question, namely given a model of \textit{impure data}, for which the sample classes are unknown, can we reconstruct $U^\star$, and if so, how?  We temporarily assume that impure training data has already been analyzed in order to generate the PDFs associated with the following definitions.  We return to the issue of modeling in Secs.\ \ref{subsec:alpha} and \ref{subsec:impprev}.

\begin{notation}
It is important to distinguish properties (e.g.\ the prevalence) of impure training populations from those of independent test populations.  The former are always used for classification, whereas the latter are the objects to which classifiers are applied.  To reduce confusion,  the symbol $\alpha$ refers to the prevalence of impure training populations, whereas $q$ refers to test populations.  
\end{notation}

\begin{definition}[Linearly Independent Populations]
Let $Q_l(\br) = Q(\br;\alpha_l)$ and $Q_h(\br)=Q(\br;\alpha_h)$ be the PDFs of two test populations for which $\alpha_l < \alpha_h$.  Then we refer to the $Q(\br;\alpha_l)$ and $Q(\br;\alpha_h)$ as the PDFs of {\bf linearly independent populations.}  When clear from context, we omit the dependence of $Q(\br;\alpha_l)$ and $Q(\br;\alpha_h)$ on $\alpha_l$ and $\alpha_h$. \label{def:linind}
\end{definition}

We can motivate the concept of linearly independent populations by recalling that the $Q(\br;\alpha_j)$ ($j\in \{l,h\}$)  are defined as
\begin{align}
Q_j(\br)=\alpha_j P(\br) + (1-\alpha_j)N(\br). \label{eq:impure}
\end{align}
Thus, by {\it formally} taking linear combinations of the $Q(\br;\alpha_j)$, we anticipate being able to recover $P(\br)$ and $N(\br)$, provided the $\alpha_j$ are different.  Observe that the case $\alpha_l=0$, $\alpha_h=1$ defines linearly independent populations that are also pure.  By analogy to linear algebra, we can treat these populations as being ``orthogonal'' in the sense that the set of classes in one is disjoint from the set of classes in the other.  These observations motivate our using the PDFs of linearly independent populations to construct $U^\star$, which suggests the following definition.
\begin{definition}[Pseudoprevalence]
Let $Q_l(\br)$ and $Q_h(\br)$ be distributions associated with linearly independent populations.  
Then we define $\delta$ to be the \textbf{pseudoprevalence} associated with the \textbf{expected impure error}
\begin{align}
\mathcal E_I(U;Q_l,Q_h,\delta) = (1-\delta)\int_{D_h}\!\!\!\! Q_l(\br) \,\,\, \d \br + \delta \int_{D_l}\!\!\!\! Q_h(\br)\,\,\, \d \br, \label{eq:imperrordef}
\end{align}
where $U=\{D_l,D_h\}$.
\label{def:pseudo}
\end{definition}

\begin{keyremark}
It is tempting to interpret Eq.\ \eqref{eq:imperrordef} as implying 
\begin{align}
Q(\br;q) = \delta Q_h(\br) + (1-\delta)Q_l(\br) \nonumber
\end{align}
for some third distribution $Q(\br;q)$, which suggests that $\delta$ is in some sense a prevalence.  One can even follow this logic further to deduce that $\delta$ is a function of $q$ sharing some, but not all of the properties of Def.\ \ref{def:prev}.  However, doing so immediately constrains our interpretation of $\delta$ in a way that makes it difficult to construct optimal classifiers.  {\it Thus, for the time being, we treat $\delta$ appearing in Eq.\ \eqref{eq:imperrordef} as a free parameter whose relationship to $q$ is yet to be deduced}.  This justifies calling $\delta$ a pseudoprevalence.  It nominally plays the same role as $q$ appearing in Eq.\ \eqref{eq:error} but is not tied to the property of a population.  In Sec.\ \ref{subsec:impclass} we construct a relationship between $\delta$ and $q$.
\end{keyremark}

The structure of Eq.\ \eqref{eq:imperrordef} suggests the following corollary to Lemma \ref{lem:optclass}.
\begin{corollary}
The sets 
\begin{subequations}
\begin{align}
D_l^\star(\delta) = \{\br: \delta Q_h(\br) < (1-\delta)Q_l(\br) \} \cup B^\star_l \label{eq:dlset} \\
D_h^\star(\delta) = \{\br: \delta Q_h(\br) > (1-\delta)Q_l(\br) \} \cup B^\star_h \label{eq:dhset}\\
B_l^\star \cup B_h^\star = B^\star = \{\br: \delta Q_h(\br) = (1-\delta)Q_l(\br)\}\label{eq:blhset}
\end{align}
\end{subequations}
define an equivalence class of partitions that minimize Eq.\ \eqref{eq:imperrordef}.  Moreover, the converse is true: any partition minimizing Eq.\ \eqref{eq:imperrordef} is in the equivalence class defined by Eqs.\ \eqref{eq:dlset}--\eqref{eq:blhset}. \label{cor:lemext}
\end{corollary}

\subsection{Classification with Impure Training Data}
\label{subsec:impclass}

We next consider how Def.\ \ref{def:pseudo} can be used for classification.  The following example provides motivation.

\begin{example}  Assume two impure populations whose prevalence values satisfy $\alpha_l < \alpha_h$.  By simple re-arrangement, it is straightforward to show that 
\begin{subequations}
\begin{align}
Q_h(\br) = Q_l(\br) &\iff  P(\br) = N(\br),\\
Q_h(\br) > Q_l(\br) &\iff  P(\br) >  N(\br),\\
Q_h(\br) < Q_l(\br) &\iff  P(\br) <  N(\br).
\end{align}
\end{subequations}
In other words, the $\delta=0.5$ and $q=0.5$ classification domains constructed from impure and pure distributions are the same, and we may use $Q_h(\br)$ and $Q_l(\br)$ as proxies for $P(\br)$ and $N(\br)$ when classifying.  Doing so yields
\begin{subequations}
\begin{align}
D_h^\star = D_p^\star(0.5) &= \{r: Q_h(\br) > Q_l(\br) \} \cup B_h^\star\\
D_l^\star = D_n^\star(0.5) &= \{r: Q_h(\br) < Q_h(\br) \} \cup B_l^\star\\
B^\star &= \{r: Q_h(\br) = Q_l(\br)\} = B_h^\star \cup B_l^\star.
\end{align}
\end{subequations}
Thus, {\it any} linearly independent population yields the same equivalence class of partitions for $q=0.5$. \label{ex:q5}
\end{example}

\begin{remark}
Example \ref{ex:q5} does not require knowledge of $\alpha_l$ and $\alpha_h$.  In retrospect, this is not surprising.  Since the $B(q=0.5)$ is the set of $\br$ for which $P(\br)=N(\br)$, all convex combinations of these PDFs remain equal on $B$.  See Fig.\ \ref{fig:convexcombo}.  However, as the next lemma demonstrates, this simplification is strongly related to the fact that $q=0.5$.   
\end{remark}

\begin{figure}\begin{center}
\includegraphics[width=8cm]{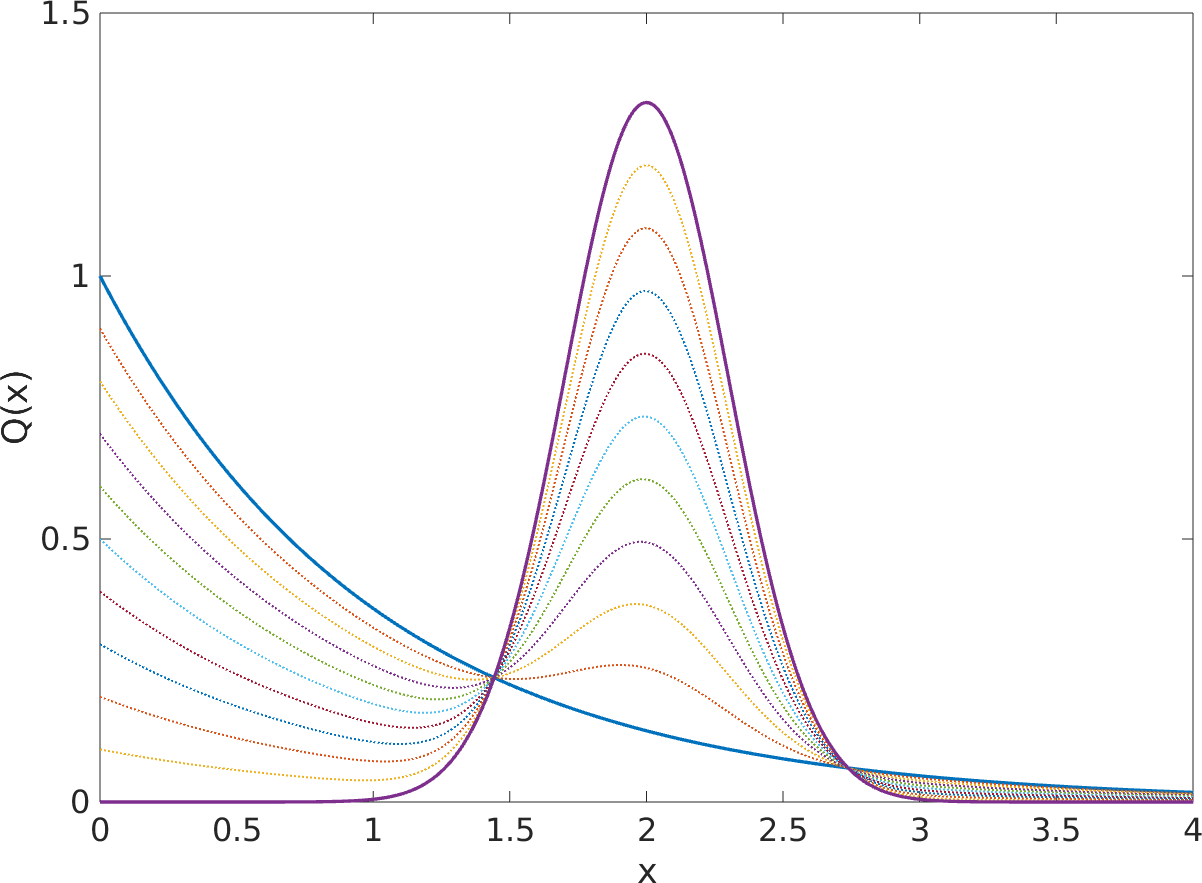}\end{center}\caption{Different convex combinations of a Gaussian and exponential distribution.  Letting $f_1(x)$ and $f_2(x)$ denote the corresponding probability densities, the figure shows 11 convex combinations $Q_j(x)=\alpha_jf_1(x)+(1-\alpha_j)f_2(x)$ for $\alpha = 0,0.1,0.2,...,1$.  Note that all convex combinations intersect at the same two points.}\label{fig:convexcombo}
\end{figure}

\begin{proposition}[Optimal Classifiers with Impure Data]\label{prop:impure}
Let $Q_l(\br)$ and $Q_h(\br)$ be the PDFs associated with impure training data.  Without loss of generality, assume that $\alpha_l < \alpha_h$.  Then the function 
\begin{align}
\delta(q) = \frac{q(1-\alpha_l) + (1-q)\alpha_l}{(\alpha_l+\alpha_h)(1-q) + (2-\alpha_l-\alpha_h)q} \label{eq:delq}
\end{align}
on the domain $q\in[0,1]$ has the additional properties that:
\begin{itemize}
\item[(a)] its range is $\delta\in [\delta_l,\delta_h]$, where
\begin{align}
\delta_l &= \frac{\alpha_l}{\alpha_l+\alpha_h} < \frac{1}{2}, & \delta_h = \frac{1 - \alpha_l}{2-\alpha_l-\alpha_h} > \frac{1}{2}; \label{eq:delbds}
\end{align}
\item[(b)] it is a strictly monotone increasing function of $q$, and hence a one-to-one mapping between $[0,1]$ and $[\delta_l,\delta_h]$;
\item[(c)] and it yields the optimal classification domains
\begin{subequations}
\begin{align}
\hspace{-5mm}D_n^\star(q)  &= \{\br: \delta(q)Q_h(\br) < (1-\delta(q))Q_l(\br) \} \cup B_l^\star \label{eq:ID0} \\
\hspace{-5mm}D_p^\star(q)  &= \{\br: \delta(q)Q_h(\br) > (1-\delta(q))Q_l(\br) \} \cup B_h^\star \label{eq:ID1}\\
\hspace{-5mm}B^\star(q) &= B_l^\star \cup B_h^\star \nonumber \\ &= \{\br: \delta(q)Q_h(\br) = (1-\delta(q))Q_l(\br) \} \label{eq:IB}
\end{align}
\end{subequations}
\end{itemize}
where $B_l^\star$ and $B_h^\star$ form an arbitrary partition of the boundary set.  In other words, Eqs.\ \eqref{eq:ID0} -- \eqref{eq:IB} define the same equivalence class of partitions as Eqs.\ \eqref{eq:Dp} and \eqref{eq:Dn}.  Moreover, for an arbitrary partition $U=\{D_l,D_h\}$ there exists a constant $C$ independent of $U$ such that
\begin{align}
(\alpha_h-\alpha_l)\mathcal E(U;q) + C &= \mathcal E_I(U;Q_l,Q_h,\delta(q))  \label{eq:linrel} \\ &= \delta(q)\!\! \int_{D_l}\!\!\! Q_h(\br) dr + (1-\delta(q))\!\! \int_{D_h}\!\!\! Q_l(\br) dr. \nonumber
\end{align}
\end{proposition}

\begin{proof}
Property (a) follows from (b) and the facts that $\delta(0)=\delta_l$ and $\delta(1)=\delta_h$.  To prove (b), it suffices to compute the derivative of $\delta(q)$ with respect to $q$.  Straightforward calculations show this to be positive (and in fact, it is proportional to $\alpha_h-\alpha_l$).    

Property (c) arises by considering inequalities of the form
\begin{align}
\delta Q_h(\br) > (1-\delta) Q_l(\br). \label{eq:meq}
\end{align}
Invoking the definition of $Q_l(\br)$ and $Q_h(\br)$, straightforward rearrangement shows that inequality \eqref{eq:meq} is equivalent to
\begin{align}
[\delta \alpha_h - (1-\delta)\alpha_l]P(\br) &>  [(1-\delta) (1-\alpha_l) -\delta(1-\alpha_h)]N(\br), \nonumber \\
\implies qP(\br) &> (1-q)N(\br), \label{eq:implied}
\end{align}
where the last line arises from Eq.\ \eqref{eq:delq}.  Identical arguments apply if the inequality sign is reversed or replaced with equality. 

Equation \ref{eq:linrel} is verified directly but substituting the definitions of $Q_h$ and $Q_l$ into Eq.\ \eqref{eq:imperrordef} and using the fact that all PDFs are normalized on $\Gamma$.  
\end{proof}

\begin{keyremark} \label{rem:merror}
The key implication of Propostion \ref{prop:impure} is that we may construct optimal classification domains for the {\it pure} problem by minimizing Eq.\ \eqref{eq:imperrordef}.  \textit{By virtue of Corollary \ref{cor:lemext}, we need not even know $Q_h$ and $Q_l$ directly}.  The logic is as follows.  Suppose we find a classifier $U^\dagger$ that minimizes Eq.\ \eqref{eq:imperrordef}.  By Corollary \ref{cor:lemext} this is in the same equivalence class as $U^\star$ given by Eqs.\ \eqref{eq:dlset}--\eqref{eq:blhset}.  But by Proposition \ref{prop:impure} this is in one-to-one correspondence with the equivalence class defined by Eqs.\ \eqref{eq:Dp}--\eqref{eq:Dp}.  Thus $\mathcal E_I(U;Q_l,Q_h,\delta(q))$ is in theory equivalent to $\mathcal E(U;q)$, which is the mathematical content of Eq.\ \eqref{eq:linrel}.

In practice the situation is more complicated as we must first determine $\alpha_l$ and $\alpha_h$.  This is addressed in the next section.  For now, however, observe the importance of the linear independence assumption.  When $\alpha_l=\alpha_h$, the relationship between $\mathcal E$ and $\mathcal E_I$ becomes vacuous; one finds that $\delta(q) \to 1/2$ for all $q$, which is consistent with Example \ref{ex:q5}.  The reader wishing to rederive many of the results herein will find that the difference $\alpha_h-\alpha_l$ appears often.  
\end{keyremark}

\begin{notation}  When expressing the optimal partition $U$ in terms of $\delta$, we use the symbols $\mathbb D_p^\star(\delta)=D_p^\star(q(\delta))$ and $\mathbb D_n^\star(\delta)=D_n^\star(q(\delta))$.  
\end{notation}

%\begin{figure*}\begin{center}
%\includegraphics[width=12cm]{placeholder.jpg}\end{center}\caption{Illustration of analysis for synthetic impure training data.  The first row shows samples drawn from Eqs.\ \eqref{eq:PDF1} and \eqref{eq:PDF2} (A), $Q_h(\br)$ defined by Eq.\ \eqref{eq:PDF3} (B), and $M_2(\br)$ defined by Eq.\ \eqref{eq:PDF4} (C).  The second row shows empirical classification boundaries associated with a quadratic model, taking $Q_h(\br)$ and $M_2(\br)$ as training data for increasing $\delta$.  The left column (D) shows the boundaries on top of data drawn from $Q_h$ and $M_2$, and the middle column (E) focuses on a region with high density of data.  Note that the boundaries only tend to overlap where there is little data.  The right column (F) shows data with true classes along with the classification boundary for $\delta=0.7$.  Note that while Eqs.\ \eqref{eq:PDF1} and \eqref{eq:PDF2} have the same support, the analysis does an excellent job of approximating a domain for which $N(\br)>0$ and $P(\br)\approx 0$.  Plots on the bottom row (G,H,I) have the same interpretation as the middle row, except that we consider decreasing values of $\delta$.  In particular, the bottom-right plot (I) compares data with true classes against the classification boundary for $\delta=0.3$, which attempts to separate $D_1(R_l^+)$ and $D_2(R_l^+)$.
%}\label{fig:pointcloud}
%\end{figure*}

\begin{example}
Consider impure training data in $\mathbb R^2$ for which
\begin{subequations}
\begin{align}
P(x,y)&=\frac{1}{\pi}e^{-\frac{(x-2)^2}{2} - 2(y-2)^2} \label{eq:PDF1} \\
N(x,y)&=\frac{1}{2\pi}e^{-\frac{x^2+y^2}{2}}\label{eq:PDF2} \\
Q_h(x,y)&=\alpha_hP(x,y)+(1-\alpha_h)N(x,y)\label{eq:PDF3} \\
Q_l(x,y)&=\alpha_lP(x,y)+(1-\alpha_l)N(x,y).\label{eq:PDF4}
\end{align}
\end{subequations}
Part I shows that the classification boundaries for such pure distributions can be constructed by considering ratios of the form
\begin{align}
\frac{N(x,y)}{P(x,y)} = \frac{q}{1-q}= \mathcal R = \frac{1}{2}e^{-\frac{x^2+y^2}{2} + \frac{(x-2)^2}{2} + 2(y-2)^2}.  \label{eq:ratioeq}
\end{align}
For $\mathcal R \in (0,\infty)$, solutions to this equation yield the sets $B^\star(q)$, which divide $\Gamma$ into the optimal domains $D_p^\star(q)$ and $D_n^\star(q)$.  By taking the logarithm of \cref{eq:ratioeq}, one finds that each set $B^\star(q)$ is the locus of points on a quadratic curve.  See also \cref{ex:gaussian}.

We can similarly express the ratio
\begin{align}
\frac{Q_l(x,y)}{Q_h(x,y)} &= \frac{\alpha_lP(x,y) + (1-\alpha_l)N(x,y)}{\alpha_hP(x,y) + (1-\alpha_h)N(x,y)} \nonumber \\
&= \gamma + \frac{\beta N(x,y)}{\alpha_hP(x,y) + (1-\alpha_h)N(x,y)} \nonumber \\
&= \gamma + \frac{\beta \mathcal R}{\alpha_h + (1-\alpha_h)\mathcal R} = \frac{\delta}{1-\delta} \label{eq:impureratioeq}
\end{align}
for some constants $\gamma < 1$ and $\beta < 1$.  Since $\mathcal R \in (0,\infty)$, the left-hand side (LHS) of the last line ranges from $\gamma$ to $\gamma + \beta/(1-\alpha_h)$, which means that $\delta$ is restricted to a proper subset of the interval $(0,1)$.  Moreover, since \cref{eq:ratioeq} and \cref{eq:impureratioeq} both only depend on $\mathcal R$, it is clear that they share the same boundary sets.  In fact, this example holds more generally, since it does not rely on the specific form of $\mathcal R$.  
\end{example}

\subsection{Estimating Prevalence of Impure Training Data}
\label{subsec:alpha}

Section \cref{subsec:impclass} assumes that the $\alpha_j$ are known, but this is rarely true in practice.  We overcome this problem by leveraging the following discontinuity.

\begin{lemma}[Boundary Jumping] \label{lem:alpha}
Assume that $U$ minimizing $\mathcal E_I(U;Q_l,Q_h,\delta)$ satisfy the optimal partition convention.  Assume also that the supports of $P(\br)$ and $N(\br)$ are partially disjoint, i.e.\ that  there exist sets $\tilde D_p$ and $\tilde D_n$ such that
\begin{subequations}
\begin{align}
P_{\tilde D_p}=\int_{\tilde D_p} P(\br)\,\, \d r &> 0, & N_{\tilde D_p}=\int_{\tilde D_p} N(\br) = 0, \\
N_{\tilde D_n}=\int_{\tilde D_n} N(\br)\,\, \d r &> 0, & P_{\tilde D_n}=\int_{\tilde D_n} P(\br) = 0.
\end{align}
\end{subequations}
Without loss of generality, let $\alpha_l < \alpha_h$.  Then for any linearly independent populations, the measures $P_{\mathbb D_p^\star(\delta)}$ and $Q_{h,\mathbb D_p^\star(\delta)}$ are discontinuous at $\delta = \delta_l$.  Similarly, $N_{\mathbb D_n^\star(\delta)}$ and $Q_{l,\mathbb D_n^\star(\delta)}$ are discontinuous at $\delta=\delta_h$.
\end{lemma}

\begin{proof}
Note first that when the optimal partition convention holds, the sets $\mathbb D_p^\star(\delta)$ are monotone decreasing as $\delta$ decreases.  That is, $\mathbb D_p^\star(\delta) \subset \mathbb D_p^\star(\delta')$ for $\delta < \delta'$; see Proposition 5.4 of Part I \cite{PartI} for justification.  Thus, by Proposition \ref{prop:impure}, within the equivalence class of partitions, the limit $\delta \downarrow \delta_l$ yields the set of $\br$ defined by the condition
\begin{align}
\lim_{\delta \to \delta_l^+}\mathbb D_{p}^\star(\delta) &= \lim_{q\to 0^+} D_p^\star(q)\nonumber \\
& = \lim_{\epsilon \to 0^+} \left\{\br:  \epsilon^{-1} N(\br) \le P(\br) \right \} \label{eq:ep}
\end{align}
We can  express this limiting set as the intersection
\begin{subequations}
\begin{align}
\lim_{q\to 0^+} D_p^\star(q) &= \bigcap_{j=1}^\infty \left\{\br: j N(\br) < P(\br) \right\} \label{eq:drl} \\
&=\{\br: N(\br)=0, P(\br) > 0\}=: D_p^\star(\delta_l^+) \label{eq:drl2}
\end{align}
\end{subequations}
To see the equality between the first and second lines, argue by contradiction.  Specifically, let there be an $\br\in \mathbb D_{p}^\star(\delta_l^+)$ such that $N(\br)>0$ and $P(\br)> 0$.  However, one can find a $j$ such that the inequality in Eq.\ \eqref{eq:drl} is violated, which means that $\br$ cannot be in every set of the intersection.  Moreover, we need not worry about points $\br$ for which $P(\br)=0$ and $N(\br)=0$, since they can be excluded from $\Gamma$.  By the assumption that the conditional PDFs are partially disjoint, we find that $\tilde D_p \subset \mathbb D_p^\star(\delta_l^+)$, and thus
\begin{align}
\lim_{\delta \downarrow \delta_l} P_{\mathbb D_p^\star(\delta)} = P_{\mathbb D_p^\star(\delta_l^+)} > 0.
\end{align}  

If in constrast, for any $\delta < \delta_l$, recourse to the definitions of $Q_l(\br)$ and $Q_h(\br)$ and Proposition \ref{prop:impure} imply that
\begin{align}
\mathbb D_p^\star(\delta) = \{r:N(\br) < 0\} = \emptyset.
\end{align}
Thus, $P_{\mathbb D_p^\star(\delta)} = 0$ for any $\delta < \delta_l$, so that
\begin{align}
\lim_{\delta \uparrow \delta_l} P_{\mathbb D_p^\star(\delta)} = 0,
\end{align}  
which yields a discontinuity for measures with respect to $P(\br)$.  By definition of the left and right limits of $\mathbb D_p^\star(\delta)$ at $\delta=\delta_l$, we see that $Q_h(\br) \propto P(\br)$ for every $\br$ in either set, so that the result also holds for measures with respect to $Q_h$.  Similar arguments yield the corresponding discontinuity for $\mathbb D_n^\star(\delta)$.  
\end{proof}

At first blush, Lemma \ref{lem:alpha} simply appears to illustrate the extent to which sets of zero measure (with respect to $Q(\br)$ in this case) need not be unique.  Viewed in the context of Eq.\ \eqref{eq:error}, this is not surprising.  The limit $\delta \downarrow \delta_l$ corresponds to $q \to 0$.  Thus the term $q\int_{D_n}P(\br) \,\,\d r$ disappears from the objective function.  In this case, we may set $D_p$ to be any set we wish, provided $N(\br)$ has zero measure there.  

For {\it impure} populations, however, we must construct the optimal partition by minimizing Eq.\ \eqref{eq:imperrordef}.  Moreover, we may always construct a family of such partitions by varying $\delta$, which we can do even when the $\alpha_j$ are unknown; see Remark \ref{rem:merror}.  Here something interesting happens.  When $\delta \downarrow \delta_l \ne 0$ (in particular, when $\alpha_l > 0$), the set $\mathbb D_p^\star(\delta) = D_p^\star(0)$ encounters the discontinuity embodied by Lemma \ref{lem:alpha}.  This happens despite neither term on the RHS of Eq.\ \eqref{eq:imperrordef} vanishing.  One therefore finds that when $Q_h(\br)$ is treated as if it were the PDF for a pure population: (i) all points in $D_p^\star(\delta_l^+)$ are associated with the class of $Q_h$ when $\delta > \delta_l$; and no points are associated with the class of $Q_h(\br)$ when $\delta < \delta_l$. % In other words, for $\delta < \delta_l$, the product $\delta \alpha_h P(\br) < (1-\delta)\alpha_l P(\br)$, even though $\alpha_h > \alpha_l$; i.e.\ even the ``more negative'' $\delta$-weighted PDF $M_0(\br)$ has more positive samples than the ``more positive'' $\delta$-weighted $Q_h(\br)$.   

To better understand this phenomenon, consider that for an arbitrary partition $U$, we can extract from Eq.\ \eqref{eq:imperrordef} the part of the expected classification error that is due to $\mathbb D_p^\star(\delta_l^+)$ when $\delta = \delta_l$.  Denote this by $\mathcal E_{\rm part}[U]$ given by
\begin{align}
\mathcal E_{\rm part}[U] &= \delta_l \int_{\mathbb D_p^\star(\delta_l^+)\cap D_n} \hspace{-15mm} Q_h(\br) \,\,\,\d r 
 +  (1-\delta_l) \int_{\mathbb D_p^\star(\delta_l^+)\cap D_p}\hspace{-15mm} Q_l(\br) \,\,\,\d r \nonumber \\
 &= \frac{\alpha_h \alpha_l}{\alpha_l+\alpha_h} \left [ \int_{\mathbb D_p^\star(\delta_l^+)\cap D_n} \hspace{-10mm} P(\br) \,\,\,\d r  
  + \int_{\mathbb D_p^\star(\delta_l^+)\cap D_p}\hspace{-10mm} P(\br) \,\,\,\d r \right ], \label{eq:impcomp}
\end{align}
where we use the fact that
\begin{align}
Q_j(\br)&=\alpha_j P(\br) && \br \in \mathbb D_p^\star(\delta_l^+). \label{eq:mjlim}
\end{align}
Equation \eqref{eq:impcomp} implies that at $\delta=\delta_l$, the $\delta$-weighted PDFs $Q_l(\br)$ and $Q_h(\br)$ are identical on $D_p^\star(\delta_l^+)$.  Thus, for $\delta < \delta_l$, $(1-\delta) Q_l(\br)$ is always greater than $\delta Q_h(\br)$ on $D_p^\star(\delta_l^+)$, and in fact, on all of $\Gamma$.  This observation and Eq.\ \eqref{eq:mjlim} motivate the following estimate of the $\alpha_j$.

\begin{construction}[System of Equations for $\alpha_j$]\label{constr:alpha}
Consider a linearly independent population, and assume that $0 < \alpha_l < \alpha_h$.  Assume also that the conditions of Lemma \ref{lem:alpha} hold.  We construct a system of equations for determining the $\alpha_j$ and show that up to degenerate cases, the solution is unique.

Note first that inverting Eq.\ \eqref{eq:delbds} for the $\alpha_j$ yields
\begin{align}
\alpha_l&=\frac{\delta_l(2\delta_h-1)}{\delta_h-\delta_l}, & \alpha_h&=\frac{(1-\delta_l)(2\delta_h-1)}{\delta_h-\delta_l}, \label{eq:alphadef} \\
1-\alpha_l &= \frac{\delta_h(1-2\delta_l)}{\delta_h-\delta_l}, & 1-\alpha_h&=\frac{(1-2\delta_l)(1-\delta_h)}{\delta_h-\delta_l}.
\end{align}
Thus, one finds
\begin{align}
\frac{\alpha_l}{\alpha_h} &= \frac{\delta_l}{1-\delta_l} & \frac{1-\alpha_h}{1-\alpha_l}=\frac{1-\delta_h}{\delta_h}. \label{eq:alpharatios}
\end{align}

Next observe that by Lemma \ref{lem:alpha}, 
\begin{subequations}
\begin{align}
Q_{j,\mathbb D_p^\star(\delta_l^+)}&=\alpha_j P_{\mathbb D_p^\star(\delta_l^+)} \label{eq:lowerMj} \\
Q_{j,\mathbb D_n^\star(\delta_h^-)}&=(1-\alpha_j) N_{\mathbb D_n^\star(\delta_h^-)} \label{eq:upperMj}
\end{align}
\end{subequations}
for $j\in \{l,h\}$.  The probability measures can be eliminated by taking ratios of the form
\begin{subequations}
\begin{align}
\frac{Q_{l,\mathbb D_p^\star(\delta_l^+)}}{Q_{h,\mathbb D_p^\star(\delta_l^+)}} = \frac{\alpha_l}{\alpha_h} = \frac{\delta_l}{1-\delta_l}, \label{eq:M0dl} \\
\frac{Q_{h,\mathbb D_n^\star(\delta_h^-)}}{Q_{l,\mathbb D_n^\star(\delta_h^-)}} = \frac{1-\alpha_h}{1-\alpha_l} = \frac{1-\delta_h}{\delta_h}. \label{eq:M1dh}
\end{align}
\end{subequations}
Equations \eqref{eq:M0dl} and \eqref{eq:M1dh} only depend on either $\delta_l$ or $\delta_h$, but not both.  Thus each equation can be solved independently.  

Consider next Eq.\ \eqref{eq:M0dl}.  The left-hand side (LHS) yields an indeterminate form for any $\delta < \delta_l$.  For $\delta>\delta_l$ observe that by linearity of integration, the ratio
\begin{align}
\frac{Q_{l,\mathbb D_p^\star(\delta)}}{Q_{h,\mathbb D_p^\star(\delta)}} = \frac{Q_{l,\mathbb D_p^\star(\delta_l^+)}+ \Delta Q_l}{Q_{h,\mathbb D_p^\star(\delta_l^+)} + \Delta Q_h}
\end{align}
for some increments $\Delta Q_l$ and $\Delta Q_h$.  But by Eq.\ \eqref{eq:ID1}, one finds that
\begin{align}
\frac{Q_{l,\mathbb D_p^\star(\delta_l^+)}+ \Delta Q_l}{Q_{h,\mathbb D_p^\star(\delta_l^+)} + \Delta Q_h} &\le \frac{\frac{\delta_l}{1-\delta_l}Q_{h,\mathbb D_p^\star(\delta^+_l)} + \frac{\delta}{1-\delta} \Delta Q_h}{Q_{h,\mathbb D_p^\star(\delta_l^+)} + \Delta Q_h} \nonumber \\
& < \frac{\delta}{1-\delta},
\end{align}
where the second line arises by observing that the first line is a weighted average of $\delta_l/(1-\delta_l)$ and $\delta/(1-\delta)$.  Thus, $\delta_l$ is the only solution to Eq.\ \eqref{eq:lowerMj}.  A similar argument applies to the case of Eq.\ \eqref{eq:M1dh}. \hfill {\small Q.E.F.}
\end{construction}

Equations \eqref{eq:M0dl} and \eqref{eq:M1dh} are given in terms of exact measures, which are generally unknown.  In light of Part I, it is useful to construct estimators of the $\alpha_j$ in terms of empirical data.  In doing so, one must address not only sampling variation due to having finite data, but also uncertainty in the models used to construct the classifier.  The latter task requires an understanding of rates of convergence the chosen classifier, which is beyond the scope of the present manuscript; see Part I for a discussion in the context of the homotopy classifier used in this manuscript \cite{PartI}.  Here we present a simple method that approximates the impact of sampling variation alone, leaving more advanced estimators for future work.  

It is necessary to introduce several auxiliary concepts.  

\begin{definition}[Linearly Independent Empirical Populations]
Let $\Omega$ be a sample space and $C(\omega)$ be a binary random variable for $\omega \in \Omega$.  Let $q\in [0,1]$, and assume the $\omega_i$ are {\it iid}.    We define 
\begin{align}
\Psi(q,s)=\left\{ \br(\omega_i):i\in \{1,...,s\}, \sum_{i=1}^s\frac{C(\omega_i)}{s} = q \right\} \label{eq:impureset} 
\end{align}
to be an {\bf empirical test population} with $s$ samples and prevalence $q$.  We further define $\Pi_{\rm impure}$ to be \textbf{linearly independent empirical populations} if
\begin{align}
\Pi_{\rm impure} = \{\Psi_l(\alpha_l,s_l),\Psi_h(\alpha_h,s_h)\}
\end{align}
is a set of distinct empirical test populations whose sample points are iid and for which $\alpha_l \ne \alpha_h$.
\end{definition}
See Part I for additional context on empirical populations \cite{PartI}.

\begin{construction}[Bayesian Approximation of $\alpha_l$ and $\alpha_h$]\label{constr:bayesian}
Let $\Pi_{\rm impure}$ be an impure training population whose prevalences $\alpha_l$ and $\alpha_h$ are unknown.  Also let $G=\{\delta_k\}$ be a grid of $\delta_k$ values satisfying $\delta_k \in (0,1)$ and $\delta_k < \delta_{k+1}$.  Assume that $\mathcal U=\{U_k^\star\}$ is a set of partitions constructed so that $U_k^\star$ minimizes the error $\mathcal E_I(U;Q_l,Q_h,\delta_k)$.  We construct a Bayesian estimate of the unknown prevalences $\alpha_j$.

Consider Eq.\ \eqref{eq:M0dl} and a corresponding estimate of $\delta_l$.  Let $j\in \{l,h\}$ The partitions $U_k^\star$ induce approximate measures
\begin{align}
\tilde Q_{j,k} = \sum_{\br \in \Psi_j(\alpha_j,s_j)} \frac{\mathbb I(\br \in \mathbb D_p^\star(\delta_k))}{s_j},
\end{align}
which are binomial random variables.  Moreover, for a fixed $k$, the $\tilde Q_{l,k}$ and $\tilde Q_{h,k}$ are independent, since the underlying populations are independent.  The Wilson score interval \cite{Wilson} yields an estimate of the confidence interval for a binomial random variable by first taking $z$ to be the z-score associated with $X$ were we able to treat $\tilde Q_{j,k}$ as a normal random variable.  That is, $|z|=2$ corresponds to a 95\% confidence interval in the estimate of the expected value of $\tilde Q_{j,k}$.  Then, the true measure is estimated to be in the range
\begin{subequations}
\begin{align}
Q_{j,k} &\in [Q_{j,k}^-,Q_{j,k}^+] \\
Q_{j,k}^\pm &= \frac{s_j}{s_j+z^2}\Bigg[\tilde Q_{j,k} + \frac{z^2}{2s_j} \nonumber \\
&\quad \pm \frac{z^2}{2s_j}\sqrt{2s_j \tilde Q_{j,k}(1-\tilde Q_{j,k}) + z^2}\Bigg].
\end{align}
\end{subequations}
By Construction \ref{constr:alpha}, we know that 
\begin{align}
\frac{Q_{l,\mathbb D_p^\star(\delta)}}{Q_{h,\mathbb D_p^\star(\delta)}} < \frac{\delta}{1-\delta}
\end{align}
for $\delta_l < \delta$.  Moreover, $Q_{l,\mathbb D_p^\star(\delta)}/Q_{h,\mathbb D_p^\star(\delta)}$ reverts to a degenerate form for $\delta < \delta_l$.  Thus, for a fixed $z$, we consider the admissible range $\Delta$ of values of $\delta_l$ to be 
\begin{align}
\Delta(z) = \left\{\delta_k: \frac{Q_{l,k}^+}{Q_{h,k}^-} \ge \frac{\delta_k}{1-\delta_k}  \right\} \label{eq:pseudoset}
\end{align} 
If $\Delta(z)$ defined by Eq.\ \eqref{eq:pseudoset} is empty, we instead define $\Delta(z)=\{0\}$.  

If $\Delta(z)=\{0\}$, this implies that $\delta_l=\alpha_l=0$, which is our estimator.  If $\Delta(z) \neq \{0\}$, let $|\Delta(z)|$ be the cardinality of $\Delta(z)$.  As our prior, we then assume that all $\delta \in \Delta(z)$ have equal probability of being $\delta_l$.

To compute a posterior distribution for $\delta_k$, observe that Eq.\ \eqref{eq:M0dl} can be expressed as 
\begin{align}
\mathcal Q_k = Q_{l,\mathbb D_p^\star(\delta)}(1-\delta) - Q_{h,\mathbb D_p^\star(\delta)}\delta = 0. \label{eq:meanzero}
\end{align}
Given the empirical data, we seek to estimate the probability that this equation is satisfied.  In particular, for each $\delta_k$, approximate
\begin{align}
Q_{j,\mathbb D_p^\star(\delta_k)} = \mathcal N\left(\tilde Q_{j,\mathbb D_p^\star(\delta_k)},\frac{\tilde Q_{j,\mathbb D_p^\star(\delta_k)}(1-\tilde Q_{j,\mathbb D_p^\star(\delta_k)})}{s_j} \right)\nonumber 
\end{align}
where $\mathcal N(\mu,\sigma^2)$ is a normal random variable with mean $\mu$ and variance $\sigma^2$.  In light of this approximation, the LHS of Eq.\ \eqref{eq:meanzero} is a normal random variable, from which we can construct the probability density $\mathcal P(\mathcal Q_k|\delta_k)$.  Using the definition of conditional probability, we then define the probability that $\delta_k$ is $\delta_l$ to be
\begin{align}
\mathcal P(\delta_k|\tilde Q_k) = \frac{ \mathcal P(\mathcal Q_k| \delta_k) \mathbb I(\delta_k \in \Delta(z))}{ \sum_k \mathcal P(\mathcal Q_k| \delta_k) \mathbb I(\delta_k \in \Delta(z)) }.
\end{align}
From this, we can compute an expected value of $\delta_l$ and corresponding confidence intervals.  A similar construction yields $\delta_h$.  Note finally that inverting Eqs.\ \eqref{eq:alphadef} yield the $\alpha_j$.  \hfill {\small Q.E.F}
\end{construction}

Constructions \ref{constr:alpha} and \ref{constr:bayesian} are limiting insofar as they require the supports of $N(\br)$ and $P(\br)$ to be partially disjoint.  In practice, this condition may not always hold; rather, it is more reasonable to assume that there exist domains $\mathcal D_+$ and $\mathcal D_-$ on which the function $R(\br) = N(\br)/P(\br) \ll 1$ and $R(\br)^{-1} = P(\br)/N(\br) \ll 1$, respectively.  In such cases, we can still construct reasonable approximations of the $\alpha_j$.  The following definition and construction makes this precise in an asymptotic sense.  

\begin{definition}
The ratio $R: \Gamma \to [0,\infty]$ defined as
\begin{align}
R(\br) = \frac{N(\br)}{P(\br)} \label{eq:relprob}
\end{align}
is the {\bf relative conditional probability} of a measurement outcome, where we interpret the situation $R(\br)= \infty$ to correspond to $P(\br)=0$ and omit from $\Gamma$ points for which $P(\br)=N(\br)=0$.
\end{definition}

\begin{notation}
We use the notation $f(x)=\mathcal O(x)$ as $x\to 0$ to mean that for some $0 < K < \infty$, there exists an $\epsilon$ such that $f(x) \le K x$  for $ 0 \le x < \epsilon$.   
\end{notation}

\begin{construction}[Asymptotic Approximation of $\alpha_j$]
Let $R_D= N_D/P_D$, and assume that there exists domains $\mathcal D_+$ and $\mathcal D_-$ with corresponding $R_+ = R_{\mathcal D_+}$ and $R_- =R_{\mathcal D_-}$ satisfying
\begin{subequations}
\begin{align}
\frac{2-\alpha_l-\alpha_h}{\alpha_l+\alpha_h}R_+ &< 1 \\
\frac{\alpha_l+\alpha_h}{2-\alpha_l-\alpha_h}R_-^{-1} & < 1.
\end{align}
\end{subequations}
These inequalities yield asymptotic estimates of $\delta_l$, $\delta_h$, $\alpha_l$, and $\alpha_h$ via Eqs.\ \eqref{eq:M0dl} and \eqref{eq:M1dh}.  In particular, define $\tilde \delta_l$ and $\tilde \delta_h$ via
\begin{subequations}
\begin{align}
\frac{Q_{l,\mathcal D_+}}{Q_{h,\mathcal D_+}} &= \frac{\alpha_l + (1-\alpha_l)R_+}{\alpha_h + (1-\alpha_h)R_+} = \frac{\tilde \delta_l}{1-\tilde \delta_l} \label{eq:Rineq1} \\
\frac{Q_{h,\mathcal D_-}}{Q_{l,\mathcal D_-}} &= \frac{1-\alpha_l + \alpha_l R_-^{-1}}{1-\alpha_h + \alpha_h R_-^{-1}} = \frac{1- \tilde \delta_h}{\tilde \delta_h}.\label{eq:Rineq2}
\end{align}
\end{subequations}
Inverting these yields 
\begin{subequations}
\begin{align}
\tilde \delta_l &= \frac{\alpha_l}{\alpha_l+\alpha_h}\left[1+ \left( \frac{2-\alpha_l-\alpha_h}{\alpha_l+\alpha_h} \right)R_+ \right]^{-1} \\
\tilde \delta_h &= \frac{1-\alpha_l}{2-\alpha_l-\alpha_h}\left [\frac{1+\frac{\alpha_l}{1-\alpha_l}R_-^{-1}}{1+\frac{\alpha_l+\alpha_h}{2-\alpha_l-\alpha_h}R_-^{-1}} \right]
\end{align}
\end{subequations}
By inequalities \eqref{eq:Rineq1} and \eqref{eq:Rineq2}, the approximations $\tilde \delta_l$ and $\tilde \delta_h$ can be expanded in powers of $R_+$ and $R_-$.  Letting $\tilde R = \max[R_+,R_-^{-1}]$, one thus finds that 
\begin{align}
\tilde \delta_l = \delta_l + \mathcal O(\tilde R) && \tilde \delta_h = \delta_h + \mathcal O(\tilde R)
\end{align}
as $\tilde R \to 0$.  By solving for the $\alpha_j$ [e.g.\ via Eq.\ \eqref{eq:alphadef}], one recovers similar asymptotic estimates for the prevalences directly.  \hfill {\small Q.E.F}
\end{construction}

The following examples illustrate the main ideas of this section.

\begin{figure}
\includegraphics[width=13cm]{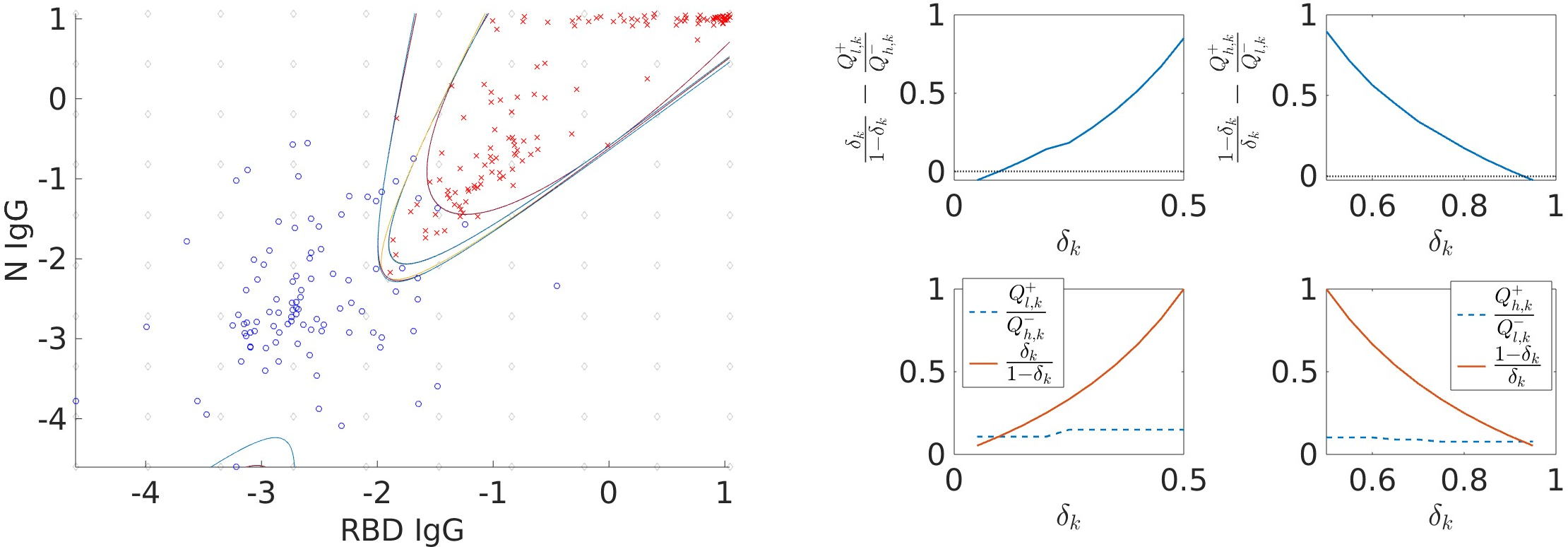}\caption{Example SARS-CoV-2 antibody measurement data and interpretation thereof.  {\it Left}: Red $\textcolor{red}{\times}$ are presumed to be positive for COVID-19 infection, whereas blue \textcolor{blue}{o} are prepandemic samples.  The classification boundaries were computed using the constrained homotopy classifier in Part I \cite{PartI}.  The associated objective function is a sum of pseudoprevalence-weighted empirical errors, with $\delta$ ranging in increments of $0.05$ from $0.05$ to $0.95$.  Note that many such boundaries have collapsed onto one another due to the finite amount of training data.  The dim black diamonds are shadow points associated with the constraints.  See Part I for more details \cite{PartI}. \textit{Right:} Various aspects of \cref{constr:bayesian} and the terms defining \cref{eq:pseudoset}.}\label{fig:sarsdata}
\end{figure}

\begin{example}[SARS-CoV-2 Assay]\label{ex:sars}
A key problem during the COVID-19 pandemic was a lack of control samples associated with positive and negative individuals.  The data from Ref.\ \cite{Raquel1} provides one such example.  The left plot of \cref{fig:sarsdata} shows typical serology measurements of two immunoglobulin g (IgG) antibodies that bind to the the SARS-CoV-2 receptor binding domain (RBD) and nucleocapsi (N).  However, the presumed positive class was based hospitalized patients with COVID-like symptoms, and orthogonal confirmation of the sample status was not available in all cases.  Thus, this data should be treated as impure, since some of the patients could have had other respiratory conditions such as influenza.  The right plot of \cref{fig:sarsdata} illustrates the results of \cref{constr:bayesian} applied to this data.  \Cref{constr:bayesian} indicates that the presumed positive samples had a prevalence of $95\%$ with a 95\% confidence interval of $[89\%,100\%]$, with $z=3$ for the prior.  This result is consistent with the understanding that the majority of the patients were COVID-19 positive based on symptoms.   Note that while the left subplots of \cref{fig:sarsdata} show the inequalities used to estimate $\alpha_l$, this quantity was set to $0$, as the corresponding samples were known to have been collected before the COVID-19 pandemic.
\end{example}

\begin{example}[Gaussian Distributions Revisited]\label{ex:gaussian}
\Cref{fig:gaussian} shows the same analysis as \cref{ex:sars}, except that the underlying PDFs are given by \cref{eq:PDF3} and \cref{eq:PDF4} with $\alpha_h=0.8$ and $\alpha_l=0.2$.  Five hundred datapoints were generated for each population.  The resulting classification boundaries are known to be quadratic in this example and can be estimated using the analysis of Part I \cite{PartI}.  The prevalence $\alpha_l$ was estimated to be $21\%$ with a 95\% confidence interval of $[8\%,23\%]$.  The prevalence $\alpha_h$ was estimated to be $89\%$, with a 95\% confidence interval of $[69\%,91\%]$.
\end{example}

\begin{figure}
\includegraphics[width=13cm]{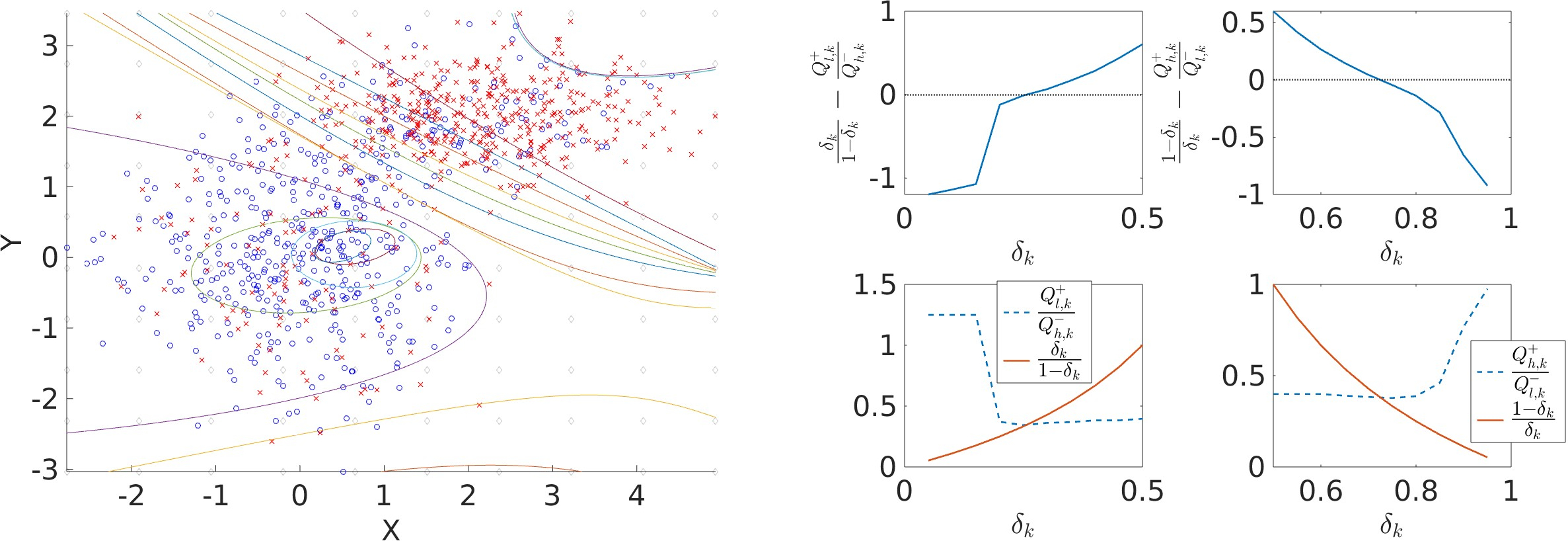}\caption{Synthetic data and interpretation thereof.  {\it Left}: Red $\textcolor{red}{\times}$ are drawn from \cref{eq:PDF3} with $\alpha_h=0.8$.  Blue \textcolor{blue}{o} drawn from \cref{eq:PDF4} with $\alpha_l=0.2$.  The classification boundaries were computed using the constrained homotopy classifier in Part I \cite{PartI}.  The associated objective function is a sum of pseudoprevalence-weighted empirical errors, with $\delta$ ranging in increments of $0.05$ from $0.05$ to $0.95$.  Note that many such boundaries have collapsed onto one another due to the finite amount of training data.  The dim black diamonds are shadow points associated with the constraints.  See Part I for more details \cite{PartI}. \textit{Right:} Various aspects of \cref{constr:bayesian} and the terms defining \cref{eq:pseudoset}.}\label{fig:gaussian}
\end{figure}

\subsection{Prevalence Estimation and Classification Revisited}
\label{subsec:impprev}

Once the $\alpha_j$ are known, we can use the impure training data to estimate the prevalence of a third test population without knowing the $P(\br)$ or $N(\br)$.  The following lemma provides needed context and is discussed in more detail in Part I \cite{PartI}.

\begin{lemma}[Class-Agnostic Prevalence Estimator]
Let $C(\omega)$ be a binary random variable.  Let $D \subset \Gamma$ be a subdomain chosen such that $|P_D-N_D|>0$.  Let there be $s$ {\it iid} random variables $\br_i$ {\rm (}$i \in \{1,2,...,s\}${\rm )} drawn from $Q(\br;q)$.  Then 
\begin{align}
\tilde q = \frac{\tilde Q_D - N_D}{P_D - N_D}, && \tilde Q_D = \frac{1}{s}\sum_{i=1}^s \mathbb I(\br_i \in D) \label{eq:prevest}
\end{align}
is an unbiased estimate of $q$ that converges in mean-square as $s\to \infty$. \label{lem:qprev}
\end{lemma}  

To estimate prevalence, it suffices to determine $P_D$ and $N_D$ in terms of the PDFs of the impure populations.  To accomplish this, define the measures
\begin{align}
Q_{j,D} = \int_D Q_j(\br) dr.
\end{align}
Next observe that by definition, 
\begin{align}
\begin{pmatrix}
Q_{l,D} \\ Q_{h,D}
\end{pmatrix}
= 
\begin{bmatrix}
\alpha_l & 1-\alpha_l \\
\alpha_h & 1-\alpha_h
\end{bmatrix}
\begin{pmatrix}
P_D \\ N_D
\end{pmatrix}, \label{eq:mateq}
\end{align}
which can be inverted to yield
\begin{align}
\begin{pmatrix}
P_D \\ N_D
\end{pmatrix} =
\frac{1}{\alpha_h - \alpha_l}
\begin{bmatrix}
\alpha_h-1 & 1-\alpha_l \\
\alpha_h & -\alpha_l
\end{bmatrix}
\begin{pmatrix}
Q_{l,D} \\ Q_{h,D}\label{eq:imateq}
\end{pmatrix}
\end{align}
Substituting empirical, Monte Carlo estimates of $Q_{l,D}$ and $Q_{h,D}$ for a suitably chosen $D$ yields corresponding empirical estimates of $N_D$ and $P_D$ \cite{montecarlo}.  See Part I for a deeper discussion of prevalence estimation, especially in the context of exmpirical data \cite{PartI}.

\begin{remark}
Observe that the determinant of matrix in Eq.\ \eqref{eq:mateq} is $\alpha_h-\alpha_l$.  When $\alpha_h=\alpha_l$,  the linear system clearly becomes ill-posed.  Thus the concept of linearly independent populations is directly tied to the structure of the matrix equation that connects mixed distributions to their pure counterparts.  
\end{remark}

\section{Discussion}
\label{sec:discussion}

\subsection{Our Interpretation of Unsupervised Learning}

Our interpretation of unsupervised learning is slightly less general than commonly accepted definitions in at least two ways \cite{unsupervised}.  First, while the true class of each sample is unknown (i.e.\ the underlying PDFs of training data are unknown), we assume there are two impure populations with different prevalence values.  Our analysis also requires \textit{a priori} knowledge that the training data is comprised of two classes.  In this sense, one might argue that we have connected aspects of linear algebra and \textit{weakly or semi-supervised} learning, although such issues may ultimately amount to semantics.  It is possible that further connections to linear algebra may clarify the concept of unsupervised learning as a generalization of supervised learning.  Such issues are discussed in more detail in the next section.

\subsection{Connection to Linear Algebra}
\label{subsec:linalg}

\Cref{def:linind} and \cref{eq:imateq} establish a connection between linear algebra and classification.  A fundamental question is the extent to which this connection can be extended.  The matrix
\begin{align}
\mathcal A = \begin{bmatrix}
\alpha_l & 1-\alpha_l \\
\alpha_h & 1-\alpha_h
\end{bmatrix}
\end{align}
suggests one route to accomplishing this.  It is clear that the columns of $\mathcal A$ are linearly independent when $\alpha_l \ne \alpha_h$, and in fact \cref{eq:mateq} can be interpreted as an alternate form of \cref{def:linind}.  Thus, we anticipate that the following definition enables generalizations of our results to multiclass settings.
\begin{definition}
Let $\mathcal A$ be an $m\times m$ right stochastic matrix (i.e.\ rows sum to one) that is invertible \cite{StochasticMatrix1}.  Let $\boldsymbol {\rm P}(\br)=[P_1(\br),...,P_m(\br)]^{\rm T}$ be a vector of distinct probability density functions.  Then we say that the distributions $\boldsymbol {\rm Q}$ corresponding to \textbf{linearly independent populations} if 
\begin{align}
\boldsymbol {\rm Q}(\br) = \mathcal A \boldsymbol {\rm P}(\br). \label{eq:genmateq}
\end{align}  
\end{definition}

The perspective of \cref{eq:genmateq} is suggestive of recent works that explored related properties classifiers in the context of linear algebra.  Reference \cite{AssayInequality} in particular demonstrated that the uncertainty in prevalence estimates associated with the case $\mathcal A=\boldsymbol 1$ (where $\boldsymbol 1$ is the identity matrix) is controlled by the largest Gershgorin radius $\rho_{\max}$ of a matrix $\boldsymbol 1- \mathbb P$ \cite{Gershgorin1,Gershgorin2}, where the elements $P_{j,k}$ of $\mathbb P$ are
\begin{align}
P_{j,k} = \int_{D_j}P_k(\br) {\rm d}\br,
\end{align}
and the $D_j$ form a partition of $\Gamma$.  Moreover, one can show that $\rho_{\max}$ is minimized by a partition whose resulting matrix $\mathbb P$ is closest to the identity in some appropriate sense; thus $\mathbb P$ can be interpreted as a type of confusion matrix, where the overlap between conditional distributions increases uncertainty in prevalence estimates.  From this perspective, one sees immediately that \cref{eq:genmateq} likewise increases uncertainty in prevalence estimates as $\mathcal A$ deviates from the identity.  While beyond the scope of this work, we anticipate that rates of convergence of estimators using impure training distributions can be deduced via the methods developed in Ref.\ \cite{AssayInequality}, and moreover, these rates should be slower than when using pure training distributions.

\subsection{Limitations and Open Questions}
\label{subsec:limitations}

A fundamental and unresolved issue in this work is to deduce the convergence properties of methods such as \cref{constr:bayesian}.  It is also likely that better estimators can be formulated, given that \cref{constr:bayesian} does not account for correlation between the measures used in \cref{eq:pseudoset}.  Addressing these issues is complicated by the fact that estimates of the optimal classification domains themselves have error due to the finite amount of data used to construct them.  Resolution of such questions is reserved for future work.

While \cref{subsec:linalg} points to analogies in linear algebra that may facilitate extensions to multiclass settings, such approaches will likely need to contend with a variation on \textit{conditioning} \cite{NA}  The matrix $\mathcal A$ hints at such issues; when $\alpha_l \approx \alpha_h$, the linear system given by \cref{eq:mateq} becomes ill-conditioned.  The finite data used to construct classifiers and estimate prevalence are likely to further compound such issues.  Thus, in multiclass settings, we anticipate that extra care must be taken to ensure that the linearly independent populations have prevalence values that are nearly orthogonal in some appropriate sense.  Alternatively, methods must be developed to stabilize numerical methods.  Such issues are also left for future work.

{\it Acknowledgements:} This work is a contribution of the National Institutes of Standards and Technology and is therefore not subject to copyright in the United States.  RB, CF, and AM were supported under the US National Cancer Institute, Grant U01 CA261276 (The Serological Sciences Network), Massachusetts Consortium on Pathogen Readiness (MassCPR) Evergrande COVID-19 Response Fund Award, and University of Massachusetts Chan Medical School COVID-19 Pandemic Research Fund.  Certain commercial equipment, instruments, software, or materials are identified in this paper in order to specify the experimental procedure adequately. Such identification is not intended to imply recommendation or endorsement by the National Institute of Standards and Technology, nor is it intended to imply that the materials or equipment identified are necessarily the best available for the purpose.

{\it Use of all data deriving from human subjects was approved by the NIST and University of Massachusetts Research Protections Offices.}

{\it Data availability:} Data associated with the SARS-CoV-2 assay is available for download as supplemental material to Ref.\ \cite{Raquel1}.  An open-source software package implementing these analyses is under preparation for public distribution.  In the interim, a preliminary version of the software will be made available upon reasonable request.

\bibliographystyle{siamplain}
\bibliography{curved}

\end{document}